\RequirePackage{amsmath}

\documentclass{fundam}

\usepackage[utf8]{inputenc}
\usepackage{latexsym}
\usepackage[T1]{fontenc}
\usepackage{amssymb}
\usepackage{latexsym}
\usepackage{fix-cm}
\usepackage{array}

\usepackage{url}
\usepackage{hyperref}
\usepackage{stmaryrd}
\usepackage{marvosym}
\usepackage{mdwmath}
\usepackage{mdwtab}

\usepackage{graphicx}
\usepackage{tikz}
\usetikzlibrary{arrows,positioning}

\usepackage{tabularx}
\usepackage{eqparbox}
\usepackage{booktabs}

\newcommand{\pc}{\mathbf{P}}

\publyear{22}
\papernumber{2102}
\volume{185}
\issue{1}

\title{Granular Directed Rough Sets, Concept Organization and Soft Clustering}
\author{A Mani\\
Machine Intelligence Unit\\
Indian Statistical Institute, Kolkata\\
203, B. T. Road, Kolkata (Calcutta)-700108, India\\
Email: \texttt{a.mani.cms@gmail.com, amani.rough@isical.ac.in}\\
Homepage: \url{https://www.logicamani.in}\\
Orcid: \url{https://orcid.org/0000-0002-0880-1035}}
\begin{document}

\maketitle
\runninghead{A Mani}{Granular Directed Rough Sets}

\begin{abstract}
Up-directed rough sets are introduced and studied by the present author in earlier papers. This is extended by her in two different granular directions in this research, with a surprising algebraic semantics. The granules are based on ideas of generalized closure under up-directedness that may be read as a form of weak consequence. This yields approximation operators that satisfy cautious monotony, while pi-groupoidal approximations (that additionally involve strategic choice and algebraic operators) have nicer properties. The study is primarily motivated by possible structure of concepts in distributed cognition perspectives, real or virtual classroom learning contexts, and student-centric teaching. Rough clustering techniques for datasets that involve up-directed relations (as in the study of Sentinel project image data) are additionally proposed. This research is expected to see significant theoretical and practical applications in related domains.
\end{abstract}


\section{Introduction}

In general rough sets granular, pointwise or abstract approximations and related rough objects  are studied \cite{am501,am240,ppm2,gc2018,pp2018,gcd2018}. These approximations can be derived from information tables or may be the result of abstractions from human (or machine) reasoning. A \emph{general approximation space} or a \emph{frame} is a pair of the form $S = \left\langle \underline{S}, R \right\rangle$ with $\underline{S}$ being a set and $R$ being a binary relation ($S$ and $\underline{S}$ \emph{will be used interchangeably throughout this paper}). $S$ is typically interpreted as a set of attributes or concepts. 

\emph{In human reasoning contexts, in particular, it happens that any two elements of $S$ is associated with ($R$-related to) another element}. The association may be determined by reasoning patterns, preference, or through decision-making guided by an external mechanism. The property of up-directedness is essentially this. In \cite{am5550}, the general not-necessarily relational rough set context enhanced with such a directedness is studied by the present author and applied to specific problems in concept learning. \emph{It should be noted that these approximation spaces are not necessarily derived from information tables, and a wider interpretation is admissible}. Among the many ways of representing knowledge, concepts, and higher order properties in these spaces, two specific granular approaches that capture forms of consequence and algebraic closure are studied in this research and usefully applied.  

From a purely mathematical perspective, the property of up-directedness (additionally referred to as directedness) of partial orders and semilattice orders is widely used in literature, it has additionally been used in studying concepts of \emph{ideals of binary relations} (see \cite{jc1977,am9204}. However, the groupoidal approach of \cite{ichlps2015,ichlweak2013} is relatively new. 

From the perspective of pedagogic content knowledge in classroom\\ \cite{bdtmp2008,am2022c,lkd15,jsr2022} or student centered teaching  contexts \cite{jrp2016,am5559}, it is very important to dynamically represent the concepts generated from multiple sources. Knowledge representation in the groupoidal approach of the present author in \cite{am5550}, and  \cite{am5559} address aspects of this, and to improve the scope they are extended with entirely different granules (and weaker assumptions) in this research. These are better suited for modeling in the context as the granules have improved closure-related properties. Moreover the property of up-directedness is involved at a first order level. In \cite{am2021c}, a rough approach to soft and hard cluster validation is proposed by the present author. Her intent is to additionally construct rough ontologies for the purpose. A specific application that adds to the possible ontologies using up-directed approximations is additionally outlined in this research.

In this research general approximation spaces, in which the relation $R$ is up-directed, are studied in relation to closure under up-directedness and choice functions in detail by the present author. Related approximations are characterized, and illustrated with examples.  It is additionally shown that the algebraic semantics of the latter is very distinct from those in which $R$ is a partial or quasi-order (\cite{jpr,sj,am501}). More specifically, it is proven that an algebraic model of pi-groupoidal approximations is an algebraic lattice with additional operations. Applications to human learning contexts, real or virtual classroom teaching, and rough clustering are additionally proposed.

In the following section, some essential background is mentioned. In the third section, the connections between groupoids and binary relations are outlined. Closure under up-directedness is used to define and characterize approximations in the following section. Pi-groupoids are introduced and pi-groupoidal approximations are investigated in the fifth section. In the sixth section, abstract examples for the previous two sections are presented. Models for pi-groupoidal approximations are explored in the next section. Following the nature of  knowledge representation though up-directed rough sets is explained. Applications to education research are done in the ninth section. In the next section, CUD and pi-groupoidal rough clustering are introduced and outlined for the analysis of Sentinel image datasets. Additional research directions are provided in the last section.

\section{Some Background}

Information tables are representations of structured data in tabular form. They are additionally referred to as descriptive or knowledge representation system in the AIML literature. In general rough sets, such tables are not absolutely essential, however, they are useful (if available). For details the reader is referred to \cite{am5559,am501}. 

An \emph{information table} $\mathcal{I}$, is a relational system of the form \[\mathcal{I}\,=\, \left\langle \mathfrak{O},\, \mathbb{A},\, \{V_{a} :\, a\in \mathbb{A}\},\, \{f_{a} :\, a\in \mathbb{A}\}  \right\rangle \]
with $\mathfrak{O}$, $\mathbb{A}$ and $V_{a}$ being respectively sets of \emph{objects}, \emph{attributes} and \emph{values} respectively.
$f_a \,:\, \mathfrak{O} \longmapsto \wp (V_{a})$ being the valuation map associated with attribute $a\in \mathbb{A}$. Values may additionally be denoted by the binary function $\nu : \mathbb{A} \times \mathfrak{O} \longmapsto \wp{(V)} $ defined by for any $a\in \mathbb{A}$ and $x\in \mathfrak{O}$, $\nu(a, x) = f_a (x)$.

Relations may be derived from information tables by way of conditions in the following form: For $x,\, w\,\in\, \mathfrak{O} $ and $B\,\subseteq\, \mathbb{A} $, $ \sigma xw $ if and only if $(\mathbf{Q} a, b\in B)\, \Phi(\nu(a,\,x),\, \nu (b,\, w),) $ for some quantifier $\mathbf{Q}$ and formula $\Phi$. The relational system $S = \left\langle \underline{S}, \sigma \right\rangle$ (with $\underline{S} = \mathbb{A}$) is said to be a \emph{general approximation space}. It should be noted that this universal feature of defining relations in general approximation spaces do not hold always in human reasoning contexts. 

In particular, if $\sigma$ is defined by the condition Equation \ref{pawl}, then $\sigma$ is an equivalence relation and $S$ is referred to as an \emph{approximation space}.
\begin{equation}\label{pawl}
\sigma xw \text{ if and only if } (\forall a\in B)\, \nu(a,\,x)\,=\, \nu (a,\, w) 
\end{equation}

\subsection{Notation}

\emph{In this research, prefix or Polish notation is uniformly preferred for relations and functions defined on a set. So instances of a relation $\sigma$ are denoted by $\sigma a b$ instead of $a \sigma b$ or $(a, b) \in \sigma$. If-then relations (or logical implications) in a model are written in infix form with $\longrightarrow$.} In Equation \ref{pawl}, \emph{if and only if} is used because the definition is not done within an obvious model.

Quantifications are universally put into braces for clarity. Thus a conditional statement of the form
$(\forall a, b, c)(\Phi(a, b, c) \longrightarrow \Psi(a, b, c))$ means that \emph{for any $a, b, c$ in the domain if $\Phi(a, b, c)$ holds then $\Psi(a, b, c)$ holds as well.} Universal quantifiers are not always omitted.

Steps in proofs are usually numbered.

\subsection{Algebraic Concepts}

For lattice theory, the reader is referred to \cite{jbnat,gg1998,gra2014,dp2002,jj}, and for partial algebras to \cite{bu,lj}.

\begin{definition}
A \emph{partial algebra} $P$ is a tuple of the form \[\left\langle\underline{P},\,f_{1},\,f_{2},\,\ldots ,\, f_{n}, (r_{1},\,\ldots ,\,r_{n} )\right\rangle\] with $\underline{P}$ being a set, $f_{i}$'s being partial function symbols of arity $r_{i}$. The interpretation of $f_{i}$ on the set $\underline{P}$ should be denoted by $f_{i}^{\underline{P}}$, however, the superscript will be dropped in this paper as the application contexts are simple enough. If predicate symbols enter into the signature, then $P$ is termed a \emph{partial algebraic system}.   
\end{definition}

Directed graphs can be associated with partial groupoids and their completions. Catalan numbers associated with the latter can be used to measure the extent of non-associativity.

\subsection{Simpler Approximations}

The following concepts were studied in earlier papers \cite{am5550,ams2020}. A summary of basic properties is presented in this subsection. 

\begin{definition}
In a general approximation space $S = \left\langle \underline{S}, R \right\rangle$ consider the following conditions:
\begin{align}
(\forall a, b )(\exists c) R ac \, \&\, Rbc  \tag{up-dir}\\
(\forall a) Raa   \tag{reflexivity}\\
(\forall a, b)(Rab \, \&\, Rba \longrightarrow a=b)   \tag{anti-sym}
\end{align}
If $S$ satisfies up-dir, then it will be said to be a \emph{upper directed approximation space}. If it satisfies all three conditions then it will be said to be a \emph{up-directed parthood space}.    
\end{definition}

For any element $a\in S$, the neighborhood $[a]$ generated by it will be the set
\[[a] = \{x :\, Rxa\}   \tag{nbd} \]

$[a]$ is the set of things that relate to $a$. It is sometimes denoted by  $R^{-1}(a)$ in the literature. The connection of up-directedness with other types of neighborhoods is additionally important, and provides a way of classifying them. The inverse neighborhood $[a]_i = \{x: \, Rax\}$ is additionally studied in \cite{ams2020}.  $[a]_i$ is the set of things that element $a$ relates to. These are investigated by the present author and others in a separate joint paper.

\begin{proposition}
It is provable that 
\begin{align*}
(\forall e, f, a)(e, f\in [a] \leftrightarrow a\in U_R (e, f)) \tag{nu1}\\
\text{If } (\forall a) \#([a]) \geq 2 \text{ then } (\forall a, b) \neq U_R(a, b) = \emptyset \text{ and conversely.} \tag{nu2}
\end{align*}
\end{proposition}

\begin{definition}\label{defp}
For any subset $A\subseteq S$, the following approximations can be defined:
\begin{align}
A^{l}\,=\, \bigcup \{[a]:\, [a]\subseteq A\}   \tag{lower}\\
A^{u}\,=\, \bigcup \{[a]:\, \exists z \in [a]\cap A\}   \tag{upper}
\end{align}
\end{definition}

\begin{theorem}\label{lup}
In a up-directed approximation space $S$, the following properties hold for elements of $\wp (S)$:
\begin{align*}
(\forall a) a^{ll} = a^{l} \subseteq a   \tag{l-id0}\\
(\forall a) a^{u} \subseteq a^{uu}   \tag{u-wid0}\\
(\forall a) a^{l}\subseteq a^{lu} \subseteq a^u   \tag{lu-inc}\\
(\forall a, b)(a\subseteq b \longrightarrow a^l \subseteq b^l)   \tag{l-mo}\\
(\forall a, b)(a\subseteq b \longrightarrow a^u \subseteq b^u)   \tag{u-mo}\\
S^l =S^u \subseteq S  \, \&\, \emptyset^l = \emptyset = \emptyset^u  \tag{bnd0}\\
(\forall a, b)(a\cup b)^u = a^u \cup b^u   \tag{u-union}\\
(\forall a, b) a^l \cup b^l  \subseteq (a\cup b)^l  \tag{l-union}\\
(\forall a, b) (a\cap b)^l \subseteq a^l \cap b^l   \tag{l-cap}\\
(\forall a, b)  (a\cap b)^u \subseteq a^u \cap b^u  \tag{u-cap}
\end{align*}
\end{theorem}

\begin{remark}
It may be noted that the upper cone of a subset $A$ (that is the set $\{b: \, (\exists a, c \in A)\, Rab \,\&\, Rcb\}$) is contained in $A^u$.  
\end{remark}

\begin{definition}
A subset $K$ is said to be an \emph{$R$-ideal} if and only if 
\[(\forall a\in K)(\forall b\in S)(Rba \longrightarrow b\in K)  \]
\end{definition}

\begin{proposition}
\begin{equation}
e f\in [a] \text{ if and only if } a\in U_R(e f)
\end{equation}
\end{proposition}

\begin{definition}
A subset $K$ is said to be an \emph{$R$-filter} if and only if 
\[(\forall a\in K)(\forall b\in S)(Rab \longrightarrow b\in K)  \]
\end{definition}

In general, the neighborhood $[x]$ is not a principal $R$-ideal. The concepts lead to other granulations.

\begin{remark}
Suppose a set $\underline{S}$ of concepts relating to a classroom lesson are given, and that some of these are vague. For any two concepts $a$ and $b$, assume that a concept $c$ that apparently contains the two exists -- this type of search for a $c$ amounts to taking decisions. Let this concept of apparent parthood be denoted by $R$. Depending on the context, the relation $R$ may be a up-directed, reflexive and antisymmetric relation. \emph{Apparent parthood} relation has been considered by the present author in \cite{am9969}. 

For two concepts $a$ and $b$, $ab = b$ may mean that $b$ fulfils the functions of $a$ in some sense, or is a preferance as in \emph{$b$ is preferable over $a$}. If, on the other hand, $ab\in U_R(a, b)$ then there is a implicit reference to a choice function in the search for a concept that fulfils the role of both $a$ and $b$. The idea of $ab$ being preferable over $a$ and $b$ is also admissible. However, note that transitivity is not assumed in all this.

For a subset of concepts $A$, the lower approximation is an aggregation of directed granules that are included in $A$. It may additionally be read as the collection of \emph{relatively definite concepts} that are attainable from $A$ (using common sense methods or through common knowledge). 
\end{remark}

\subsection{Approximation Spaces and Groupoids}\label{apprsp}

It should be noted that up-directedness is not essential for a relation to be represented by groupoidal operations. The following construction that differs in part from the above strategy can be used for partially ordered sets as well, and has been used by the present author in \cite{amdsc2016,am909} in the context of knowledge generated by approximation spaces. The method relates to earlier algebraic results including \cite{jjm,jj1978,kt1981,fjjm}. The groupoidal perspective can be extended for quasi ordered sets.

If $S = \left\langle \underline{S}, R \right\rangle$ is an approximation space, then define (for any $a, b\in S$)  
\begin{equation}
a\cdot b \,=\, \left\{
\begin{array}{ll}
a,  & \text{if } Rab \\
b, &  \text{if } \neg Rab 
\end{array}
\right. 
\end{equation}

Relative to this operation, the following theorem (see \cite{jjm}) holds:

\begin{theorem}\label{ab}
$\left\langle S,\, \cdot \right\rangle$ is a groupoid that satisfies the
following axioms (braces are omitted under the assumption that the binding is to the left,
e.g. '$abc$' is the same as '$(ab)c$'):
\begin{align*}
{x x = x} \tag{E1}\\
{x (a z) = (x a) (x z) } \tag{E2}\\
{x a x = x} \tag{E3}\\
{azxauz = auz } \tag{E4}\\
{u(azxa)z = uaz } \tag{E5}
\end{align*}
\end{theorem}

\begin{theorem}
The following are consequences of the defining equations of $\mathbb{E}_{0}$ (from \textbf{E1,E2,E3}): 
\begin{align*}
{x(ax) = x ;\;  x(xa) = xa ;\;  (xa)a = xa } \\
{x(xaz) = x(az) ;\;  (xz)(az) = xz ;\;  (xa)(zx) = xazx } \\
{xazxa = xa;\;  xazaz = xaz;\;  xcazaxa = xaza} \\
{(xazx)(za) = x(za) ;\;  x(az)a = xaza ;\;  (xaz)(ax) = (xza)(zx) ;\; xazxz = xzaz.} \\
{(\forall x)(ex=ea \longrightarrow x=a) \;\equiv \; (\forall x) xe = e } 
\end{align*}
\end{theorem}

\section{Groupoids and Binary Relations}

Under certain conditions, groupoidal operations can correspond to binary relations on a set. The problem of rewriting the semantic content of binary relations of different kinds using total or partial operations has been of much interest in algebra (for example \cite{rp90,ichl2011}). Results on using partial operations for the purpose are of more recent origin \cite{ichlps2015,icl2013}. 

All binary relations can be read as partial groupoidal operations in a perspective (\cite{ichlps2015}) and therefore all general approximation spaces can be transformed into partial groupoids \cite{am5550}. In this subsection known results for groupoids are stated for convenience.


Let $S=\left\langle \underline{S}, R \right\rangle $ be a relational system, define 
\begin{equation}
U_R (a, b) = \{x :\, Rax \,\&\, Rbx\} 
\end{equation}
Further, let 
\begin{equation}
U_R(S) = \{U_R(a, b): \, a, b\in S\} \tag{updir-sets}             
\end{equation}
$S$ is said to be \emph{up-directed} if and only if $U_R (a, b)$ is never empty. That is,
\begin{equation}
(\forall a, b) \, \neg U_R (a, b) = \emptyset  \tag{up-directed}
\end{equation}

Taking \[L_R (a, b) = \{x :\, Rxa \,\&\, Rxb\}, \] the dual concept of $S$ being \emph{down-directed} can be defined 

Up-directed relations can be represented with triples corresponding to its defining condition.
\begin{example}
Let $\underline{S} = \{1, 2, 3, 4, 5 \}$ and $R$ be defined by 
\[\begin{bmatrix}
114 &225 & 332 & 444 & 552\\
123 & 234 & 341 & 451 & - \\
134 & 242 & 354 & - &  - \\
145 & 251 & - & - & -  \\
151 & - & - & - & -    
\end{bmatrix} \]

An entry of the form $abc$ is intended to correspond to $Rac$ and $Rbc$. 
$\left\langle \underline{S}, R \right\rangle$ is then a up-directed approximation space.

$U_R(1, 2) = \{3, 4, 5\}$, for example.
\end{example}

\begin{definition}\label{updg}
If a relational system is up-directed, then it corresponds to a number of groupoids defined by 
\[(\forall a, b )\, ab = \left\lbrace  \begin{array}{ll}
 b & \text{if } Rab\\
 c & c\in U_R(a, b) \,\&\, \neg Rab\\
 \end{array} \right. \tag{updg1}\]
\end{definition}
These are studied in \cite{icl2013}. The collection of groupoids satisfying the above condition will be denoted by $\mathfrak{B}(S)$ and an arbitrary element of it will be denoted by $\mathsf{B}(S)$. It may be noted that \emph{up-directed sets} (partially ordered sets that are up-directed) and related constructions are well-known in topology and algebra, however, the specific association of up-directedness mentioned is new.

\emph{Join directoids} \cite{jjq90} are groupoids of the form $S$ that admit of a partial order relation $\leq$ that satisfies $(\forall a, b)\, a, b \leq ab$ and if $\max\{a, b\}$ exists then $ab = \max \{a, b\}$.
Clearly the results of \cite{icl2013} may additionally be read as a severe generalization of known results for join directoids. It may additionally be noted that lambda lattices (that are commutative join and meet directoids) are related special cases (see \cite{sva,am105}).

\begin{theorem}[\cite{icl2013}] 
For a groupoid $A$, the following are equivalent
\begin{itemize}
\item {A up-directed reflexive relational system $S$ corresponds to $A$}
\item {$A$ satisfies the equations \[aa = a \, \&\, a(ab) = b(ab) = ab\]}
\end{itemize}
\end{theorem}

\begin{definition}
If $A$ is a groupoid, then two relational systems corresponding to it are $\Re (A) = \left\langle \underline{A}, R_A \right\rangle$ and $\Re^* (A) = \left\langle \underline{A}, R_A^* \right\rangle$ with 
\begin{align*}
R_A = \{(a, b):\, ab = b\}\\
R_A^* = \bigcup \{(a, ab),\,(b, ab)\}
\end{align*}
\end{definition}

\begin{theorem}[\cite{icl2013}]
\begin{itemize}
\item {If $A$ is a groupoid then $\Re^* (A)$ is up-directed.}
\item {If a groupoid $A\models a(ab) = b(ab) = ab$ then $\Re (A) = \Re^* (A)$.}
\item {If $S$ is an up-directed relational system then $\Re (\mathsf(B) (S)) = S$.}
\end{itemize}
\end{theorem}

\begin{theorem}[\cite{icl2013}]
 If $S = \left\langle \underline{S}, R \right\rangle$ is a up-directed relational system, then all of the following hold:
 \begin{itemize}
\item {$R$ is reflexive if and only if $\mathsf{B}(S) \models aa = a$.}
\item {$R$ is symmetric if and only if $\mathsf{B}(S) \models (ab)a = a$.}
\item {$R$ is transitive if and only if $\mathsf{B}(S) \models a((ab)c) = (ab)c$.}
\item {If $\mathsf{B}(S) \models ab = ba$ then $R$ is antisymmetric.}
\item {If $\mathsf{B}(S) \models (ab)a = ab$ then $R$ is antisymmetric.}
\item {If $\mathsf{B}(S) \models (ab)c = a(bc)$ then $R$ is transitive.}
\end{itemize}
\end{theorem}
 
Morphisms between up-directed relational systems are preserved by corresponding groupoids. A \emph{relational morphism}  (as in \cite{mal}) from a relational system 
 $S = \left\langle \underline{S}, R \right\rangle$ to another $K = \left\langle \underline{K},Q  \right\rangle$ is a map $f: S \longmapsto K$ that satisfies \[(\forall a, b)\, (Rab \, \longrightarrow Qf(a)f(b)).\]  $f$ is said to be \emph{strong} if it satisfies \[(\forall c, e\in Q)(\exists a, b\in S )\, Qf(a)f(b)\, \&\, f(a) = c, \&\, f(b) = e\]

\section{CUD-Approximations}\label{adv}

The connection between up-directedness and approximations can be explored from a higher order or an extended pseudo-closure perspective. Additionally, the latter can be interpreted (explained below) from a distributed cognition perspective.

\begin{definition}\label{dcogni}
In a up-directed approximation space $S$, for a given $A\subset S$, the \emph{i-distributed cognitive neighborhood} (idc-nbd) $\nu_{A}(x)$ of an element $x\in S$ relative to $A$ will be the set 
\begin{equation*}
\nu_A(x) =  \{z:\, (\exists h\in A) Rhz \,\&\, Rxz\}
\end{equation*}
\end{definition}

\begin{definition}\label{dcogn}
In a up-directed approximation space $S$, for a given $A\subset S$, the \emph{distributed cognitive neighborhood} (dc-nbd) $\eta_{A}(x)$ of an element $x\in S$ relative to $A$ will be the set 
\begin{equation*}
\eta_A(x) =  \{z:\, (\exists h\in A) Rhx \,\&\, Rzx\}
\end{equation*}
\end{definition}

\begin{proposition}
The following properties hold for idc-nbds in the context of Definition \ref{dcogni}:
\begin{align}
(\forall x) \nu_S (x) \neq \emptyset   \tag{idcn-ne}\\
(\forall x) (A\subset B \subseteq S \longrightarrow \nu_A (x) \subseteq \nu_B (x) )  \tag{idcn-mo}\\
(\forall x) \nu_S(x) \subseteq [x]_{-i}    \tag{dcn-mo}
\end{align}
\end{proposition}

\begin{proposition}
The following properties hold for dc-nbds in the context of Definition \ref{dcogn}:
\begin{align}
(\forall x) (A\subset B \subseteq S \longrightarrow \eta_A (x) \subseteq \eta_B (x) )  \tag{idcn-mo}\\
(\forall x) \eta_S(x) \subseteq [x]    \tag{idcn-mo}
\end{align}
\end{proposition}

\begin{definition}
A subset $A$ of a general approximation space $S$ will be said to be \emph{closed under up-directedness} (CUD) if and only if it satisfies
\begin{equation}
(\forall a, b\in A \exists c\in A)\, Rac \& Rbc
\end{equation}
\end{definition}

It can be checked that in general, neighborhoods of the form $[x], [x]_{i}, \eta_A(x) $, and $\nu_A(x)$ are not CUD sets (even in a up-directed approximation space). 

\begin{definition}
If $S$ is a general approximation space, then a partial map $\eth: \wp (S) \longmapsto \wp (S)$ will be said to be a \emph{up-directed partial closure operator} provided it satisfies 
\begin{enumerate}
 \item {$\forall A\in \wp(S)\, A \subseteq \eth (A)$}
 \item {For each $A\in \wp(S)$, $\eth (A)$ is a minimal CUD set containing $A$  with respect to set inclusion.}
\end{enumerate}
\end{definition}

\begin{example}
In any unbounded non directed lattice, $\eth$ will not be a total operation.    
\end{example}

\begin{theorem}\label{ccud}
If $S$ is a directed general approximation space, then $\eth$ is a well defined cautious closure operator. 
That is it satisfies:
\begin{align}
(\forall A \in \wp(S)) A\subseteq \eth (A) \tag{Inclusion}\\
(\forall A, B\in \wp(S))(A\subset B \subseteq \eth(A) \longrightarrow \eth (A) \subseteq \eth (B))   \tag{cmo}\\
(\forall A \in \wp(S)) \eth \eth (A) = \eth (A) \tag{Idempotence}\\
\eth(\emptyset) = \emptyset ; \, \eth(S) = S   \tag{Bot; Top}
\end{align}
\end{theorem}

\begin{proof}
If $A$ and $B$ are subsets of $S$, and $A\subset B \subseteq \eth(A)$, suppose that  $\eth (B) \subset \eth(A)$. However, by definition $\eth(B)$ is a minimal CUD-set containing $B$ and it must additionally be a CUD-set containing $A$. This contradicts the fact that $\eth(A)$ is a minimal CUD-set containing $A$. Therefore, it must be that $\eth(A)\subseteq \eth(B)$. So $\eth$ is cautious monotonic (cmo).

Further $\eth$ is idempotent because a minimal (and the smallest) CUD-set containing a given CUD-set $\eth(A)$ must be the same CUD-set $\eth(A)$. In other words, $\eth\eth(A) = \eth(A)$.

Since $\emptyset$ and $S$ are CUD-sets, $\eth(\emptyset) = \emptyset$ and $\eth(S) = S$.

\end{proof}

\begin{remark}
However, $\eth$ is neither additive nor does it preserve intersections in general. Additionally, if $A$ and $B$ are CUDS in the power set $\wp(S)$, then it need not be that $A\cup B$ or $A\cap B$ are CUD sets as well.
\end{remark}

CUD-sets are the most natural of possible granules or higher-order neighborhoods generated by a up-directed relation. Let the set of all CUD sets of $S$ generated by a relation $R$ be $\mathcal{C}_R(S)$.  The following proposition follows from the induced order.

\begin{proposition}
If $S$ is up-directed, then $\mathcal{C}_R(S)$ under the set inclusion order forms a bounded partially ordered set. 
\end{proposition}

\begin{definition}\label{cudall}
For any $A, B\in \mathcal{C}_R(S)$ let 
\begin{equation*}
A\oplus B := \eth(A\cup B) \tag{oplus} 
\end{equation*}
\begin{equation*}
A\odot B:= \eth(A\cap B) \tag{odot}
\end{equation*}
The algebraic system $\left\langle \mathcal{C}_R(S), \subseteq, \oplus, \odot, \emptyset, S  \right\rangle$ will be called a \emph{CUD-algebraic system} (CUDAS).
\end{definition}

\begin{theorem}
 In the context of Definition \ref{cudall}, the CUDAS $\left\langle \mathcal{C}_R(S), \subseteq, \oplus, \odot, \emptyset, S \right\rangle$ is well-defined and satisfies all of the following:
\begin{align*}
\odot \text{ is commutative and idempotent}   \tag{ic-odot}\\
\oplus \text{ is commutative and idempotent}   \tag{ic-oplus}\\
(\forall A, B, C, E)(A\subseteq B \,\& C\subseteq E \, \&\, B\cup E \subseteq  A\oplus C \longrightarrow A\oplus C\subseteq B\oplus E )   \tag{cmo-plus}\\
(\forall A, B, C, E)(A\subseteq B \,\& C\subseteq E \&\, B\cap E\subseteq  A\odot C\longrightarrow A\odot C\subseteq B\odot E )   \tag{cmo-dot}\\
(\forall A, B) A \subseteq A\oplus B   \tag{inclusion+}\\
(\forall A, B) A\odot B \subseteq A  \tag{inclusiondot}
\end{align*}
\end{theorem}

\begin{proof}
$\cup$ and $\cap$ are commutative and idempotent operations on the power set $\wp (S)$. 
Because $\eth \eth(A) = \eth(A)$ for an element $A$ of $\mathcal{C}_R(S)$, the properties \textsf{ic-dot} and \textsf{ic-oplus} hold.

\textsf{cmo-plus} and \textsf{cmo-dot} 
If $A$ and $C$ are CUD-sets, then $A\cup B \subseteq \eth(A\cup B)$. $A\oplus C$ includes $A$, $C$ and elements that ensure the CUD property. Such elements should necessarily be contained in $B\oplus E$

\end{proof}

The above machinery is useful for studying the following new approximations that take $\mathcal{C}_R(S)$ as the set of granules.

\begin{definition}\label{lucd}
In a directed approximation space $S$, the \emph{lower and upper  CUD approximations} ($l_{cd}$ and $U_{cd}$ respectively) of a subset $A$ can be defined using the granulation $\mathcal{C}_R(S)$ as follows ($\Upsilon$ being the union of minimal elements of the collection under set inclusion):
\begin{align*}
A^{l_{cd}} = \bigcup\{ H : \, H\in \mathcal{C}_R(S) \,\&\,H\subseteq A \}   \tag{lcd}\\
A^{u_{cd}} = \Upsilon\{ H : \, H\in \mathcal{C}_R(S) \,\&\,H\cap A\neq \emptyset \}    \tag{ucd}
\end{align*}
\end{definition}

The restricted union $\Upsilon$ is required for ensuring non triviality. The following properties can be deduced.

\begin{theorem}\label{cdbas}
In the context of Def \ref{lucd}, the approximations satisfy all of the following (for any two subsets $A$ and $B$ of $S$):
\begin{align*}
A^{l_{cd}}\subseteq A \subseteq A^{u_{cd}}   \tag{cdInclusion}\\
A^{l_{cd} l_{cd}} = A^{l_{cd}}   \tag{lcdId}\\
A^{u_{cd}} \subseteq A^{u_{cd} u_{cd}}    \tag{ucdpId}\\
A^{l_{cd}} \subseteq A^{l_{cd} u_{cd}}    \tag{lucdpId}\\
A^{u_{cd} l_{cd}} = A^{u_{cd}}   \tag{ulcdId}\\
(A\subseteq B \longrightarrow A^{l_{cd}} \subseteq B^{l_{cd}})   \tag{lcdmo}\\
(A\subseteq B \longrightarrow A^{u_{cd}} \subseteq B^{u_{cd}})   \tag{ucdmo}\\
A^{l_{cd}} \cup B^{l_{cd}} \subseteq (A\cup B)^{l_{cd}}   \tag{lcdsadd}\\
A^{u_{cd}} \cup B^{u_{cd}} \subseteq (A\cup B)^{u_{cd}}   \tag{ucdsadd}\\
(A\cap B)^{l_{cd}} \subseteq A^{l_{cd}} \cap B^{l_{cd}}   \tag{lcdsmul}\\
(A\cap B)^{u_{cd}} \subseteq A^{u_{cd}} \cap B^{u_{cd}}   \tag{ucdsmul}\\
{\emptyset}^{l_{cd}} = \emptyset = {\emptyset}^{u_{cd}}   \tag{cdbottom}\\
S^{l_{cd}} = S = S^{u_{cd}} \tag{cdtop}
\end{align*}
\end{theorem}
\begin{proof}
By definition, $A^{l_{cd}}\subseteq A$. Now every subset of $A$ is contained in some subset of $S$ that is closed under directedness. Because the union of all granules is $S$, it follows that $A \subseteq A^{u_{cd}} $.

$A^{l_{cd} l_{cd}} = A^{l_{cd}}$ follows from the fact that $A^{l_{cd}} $ is a union of granules (that are closed under directedness). 

$A^{u_{cd} l_{cd}} = A^{u_{cd}}$ follows from the fact that $A^{u_{cd}} $ is a union of granules (that are closed under directedness), and the cd-lower approximation of such a set must be the same.

If $x\in A^{u_{cd}}$, then it is in a granule $H\in \mathcal{C}_R(S)$ and $H\subseteq A^{u_{cd}}$. But this means $x\in H \subseteq A^{u_{cd} u_{cd}}$. Therefore $A^{u_{cd}} \subseteq A^{u_{cd}u_{cd}}$ holds. In general, the converse inclusion will not hold. The proof of $A^{l_{cd}} \subseteq A^{l_{cd} u_{cd}}$ is similar.

If $X$ is an element of $\mathcal{C}_R(S)$, and $X\subseteq A\subseteq B$, then $X$ is a subset of the cd-lower approximation of $A$ and $B$. So \textsf{lcdmo} follows. 

If $X$ is an element of $\mathcal{C}_R(S)$, and $X\cap A\neq \emptyset$ and $A\subseteq B$, then $X\cap B\neq \emptyset$. Note that the restricted union $\Upsilon$ is at least as big for $B$. This ensures \textsf{ucdmo}.

If $x\in A^{l_{cd}} \cup B^{l_{cd}}$, then there exists a $X\in \mathcal{C}_R(S)$ such that $x\in X\subseteq A$ or $x\in X\subseteq B$. In either case $x\in X\subseteq A\cup B$ holds. Therefore  \textsf{lcdsadd} holds. 

If $x\in A^{u_{cd}} \cup B^{u_{cd}}$, then there exists a $X\in \mathcal{C}_R(S)$ such that $x\in X$ such that $X\cap A\neq \emptyset$ or $X\cap B\neq \emptyset$. In either case, $ X\cap (A\cup B)\neq \emptyset$ holds. Therefore  \textsf{ucdsadd} holds. 

If $x\in (A \cap B)^{l_{cd}}$, then there exists a $X\in \mathcal{C}_R(S)$ such that $x\in X\subseteq A\cap B$. This means $X\subseteq A$ and $X\subseteq B$. Therefore, \textsf{lcdsmul} holds. 

The rest of the proof is left to the reader.
\end{proof}

\begin{definition}\label{roucud}
In the above context, (denoting the boundary $A^{u_{cd}}\setminus A^{l_{cd}}$ of a subset $A$ by $A^{b_{cd}}$ ) by a \emph{cud-rough tuple} will be meant a 3-tuple of the form 
\[(A^{l_{cd}}, A^{u_{cd}}, A^{b_{cd}}) \]
The collections of all such cud-rough triples will be denoted by $\mathcal{CD}(S)$.
\end{definition}

\begin{definition}\label{cdss}
A subset $A\in\wp(S)$ will be said to be a \emph{cud-subset} of $B\in \wp(S)$ (in symbols, $A\sqsubseteq_{cd} B$) if and only if 
\begin{equation*}
A^{l_{cd}} \subseteq B^{l_{cd}} \,\&\,  A^{u_{cd}} \subseteq B^{u_{cd}}    \tag{cud-ss} 
\end{equation*}
Further, they will be said to be \emph{cud-roughly equal} (in symbols, $A\approx_{cd} B$) if they  satisfy the condition \textsf{cd-ro eq}::
\begin{equation*}
A^{l_{cd}} = B^{l_{cd}} \,\&\,  A^{u_{cd}} = B^{u_{cd}}    \tag{cd-ro eq} 
\end{equation*}
\end{definition}

It is easy to see that 
\begin{theorem}
In the context of definition \ref{cdss}, $\sqsubseteq_{cd}$ is a quasi order and $\approx_{cd}$ is an equivalence relation on $\wp(S)$. 
\end{theorem}

\begin{proof}
$\sqsubseteq_{cd}$ is reflexive. It is transitive because for any $A, B, C\in \wp(S)$,
$A^{l_{cd}} \subseteq B^{l_{cd}} \,\&\,  A^{u_{cd}} \subseteq B^{u_{cd}}$ and $B^{l_{cd}} \subseteq C^{l_{cd}} \,\&\,  B^{u_{cd}} \subseteq C^{u_{cd}} $ implies \[A^{l_{cd}} \subseteq C^{l_{cd}} \,\&\,  A^{u_{cd}} \subseteq C^{u_{cd}} \]
\end{proof}

\section{Pi-Groupoids}

The non-uniqueness of the groupoidal operation associated with a up-directed relation is addressed here. 

For a given $a, b\in S$ , let $U_R^m(a, b)$ be a maximal subset of $U_R(a, b)$ that satisfies
\[(\forall e, f\in U_R^m (a, b) \exists g\in U_R(a, b)\setminus U_R^m (a, b)) Reg \& Rfg\]
$U_R^m(a, b)$ will be called a set of \emph{pseudo joins} of $a$ and $b$. $U_R^m(a, b)$ is essentially the set of potential values of the join of the elements $a$ and $b$.

\begin{proposition}
If $R$ is a join semilattice order, then for every $a, b\in S$, $U_R^m(a, b)$ is a singleton. 
\end{proposition}

The map $\pi$, defined below, is intended to be a choice function that selects a unique upper bound from the set of upper bounds of $a$ and $b$. 

\begin{definition}
A map $\pi: U_R(S) \longmapsto S$ will be called a \emph{upper choice function} if for all $a$ and $b$ in $S$ $\pi(U_R(a, b)) \in U_R(a, b)$ and a \emph{pseudo-join choice function} if it satisfies the condition $\pi(U_R(a, b)) \in U_R^m(a, b)$.
\end{definition}

\begin{definition}\label{pigrpd}
If a relational system is up-directed, then it corresponds to a \emph{Pi-groupoid} defined by (for a pseudo-join choice function $\pi$)
\[(\forall a, b )\, ab = \left\lbrace  \begin{array}{ll}
 b & \text{if } Rab\\
 \pi (U_R(a, b)) & \text{if } \neg Rab\\
 \end{array} \right. \tag{upidg} \]
\end{definition}

Pi-groupoids can be used to study many groupoidal operations or to algebraically model a specific perspective implicit in the generalities expressed by the up-directed relation. The functionality gain is extreme and described in the next subsection.

\subsection{Pi-Groupoidal Approximations}

Pi-groupoids generated by a directed approximation space are groupoids that can additionally be provided with nice granulations for algebraically computing approximations. For the granulation, it is natural to use subgroupoids and this additionally amounts to restricting the construction of CUD-approximations by special choices.

\begin{definition}\label{pigra}
On a pi-groupoid $S$, relative to the granulation $\mho = Su(S)$ (the algebraic lattice of subgroupoids of $S$), the \emph{pi-lower} ($l_\pi$) and \emph{pi-upper} ($u_\pi$) approximations of a subset $A\subseteq S$ will be as follows (the smallest subgroupoid of $S$ containing a set $A$ is $\mathbf{Sg}(A)$):
\begin{align*}
A^{l_\pi} = \bigcup \{X: X\in \mho\, \& \,X\subseteq A\}   \tag{lower-$\pi$}\\
A^{u_\pi} = \mathbf{Sg}(A)   \tag{upper-$\pi$}
\end{align*}
\end{definition}

The boundary of $A$ in the above definition will be taken to be $A^{b_\pi} = A^{u_\pi}\setminus A^{l_\pi}$. Because the granules are subgroupoids, it follows that 

\begin{proposition}
The pi-lower approximations are unions of subgroupoids, while the pi-upper approximations are subgroupoids. 
\end{proposition}

Pi-approximations satisfy substantially better properties in comparison to cd-approximations. 

\begin{theorem}\label{pi9}
Pi-lower and pi-upper approximations satisfy the following (for any two subsets $A$ and $B$ of a pi-groupoid $S$
\begin{align*}
A^{l_{\pi}}\subseteq A \subseteq A^{u_{\pi}}   \tag{piInclusion}\\
A^{l_{\pi} l_{\pi}} = A^{l_{\pi}}   \tag{lpiId}\\
A^{u_{\pi}} = A^{u_{\pi} u_{\pi}}    \tag{upipId}\\
A^{l_{\pi}} = A^{l_{\pi} u_{\pi}}    \tag{lupipId}\\
A^{u_{\pi} l_{\pi}} = A^{u_{\pi}}   \tag{ulpiId}\\
(A\subseteq B \longrightarrow A^{l_{\pi}} \subseteq B^{l_{\pi}})   \tag{lpimo}\\
(A\subseteq B \longrightarrow A^{u_{\pi}} \subseteq B^{u_{\pi}})   \tag{upimo}\\
A^{l_{\pi}} \cup B^{l_{\pi}} \subseteq (A\cup B)^{l_{\pi}}   \tag{lpisadd}\\
A^{u_{\pi}} \cup B^{u_{\pi}} \subseteq (A\cup B)^{u_{\pi}}   \tag{upisadd}\\
(A\cap B)^{l_{\pi}} \subseteq A^{l_{\pi}} \cap B^{l_{\pi}}   \tag{lpismul}\\
(A\cap B)^{u_{\pi}} \subseteq A^{u_{\pi}} \cap B^{u_{\pi}}   \tag{upismul}\\
{\emptyset}^{l_{\pi}} = \emptyset = {\emptyset}^{u_{\pi}}   \tag{pibottom}\\
S^{l_{\pi}} = S = S^{u_{\pi}} \tag{pitop}
\end{align*}
\end{theorem}
\begin{proof}
 The proof of this theorem is dictated by the algebraic closure operator $\mathbf{Sg}$, though the points in theorem \ref{cdbas} applies to these cases as well.
 
 For example, the subgroupoid generated by a subgroupoid is the same subgroupoid. So properties \textsf{upipId, luipId}, and \textsf{ulpiId} hold.
 
 The properties \textsf{upimo, upisadd} and \textsf{upismul} are essentially about subgroupoids generated by a subset of a groupoid through the algebraic closure operator $\mathbf{Sg}$.
\end{proof}

The approximations have nice modal properties, though they fail in relation to set operations.

\begin{theorem}\label{sappr}
In the above context,
\[ A^{l_\pi} \subseteq \mathbf{Sg}(A^{l_\pi}) \subseteq \mathbf{Sg}(A) = A^{u_\pi} \]
\end{theorem}

\begin{proof}
$A^{l_\pi}$ is a subset of the subgroupoid generated by itself. Therefore $A^{l_\pi} \subseteq \mathbf{Sg}(A^{l_\pi})$. $\mathbf{Sg}(A^{l_\pi}) \subseteq \mathbf{Sg}(A) $ holds because $Sg$ is an algebraic closure operator on $\wp(S)$. 
\end{proof}

In general, $\mathbf{Sg}(A^{l_\pi}$ is not a subset of $A$. So the lower approximation cannot be replaced by it.

All of the above suggests the next definition.
\begin{definition}\label{roupi}
In the above context, by a \emph{pi-groupoidal rough tuple} (pg-rough tuple) will be meant a 3-tuple of the form 
\[(A^{l_\pi}, \mathbf{Sg}(A^{l_\pi}), A^{u_\pi})\]
An \emph{algebraically constructed pi-groupoidal rough tuple} (acpg-rough tuple ) will be pairs of the form  
\[( \mathbf{Sg}(A^{l_\pi}), A^{u_\pi} ))\]
The collections of all such pg-rough and acpg-rough tuples will respectively be denoted by 
$\mathcal{PG}(S)$ and $\mathcal{ACP}(S)$.
\end{definition}

\begin{definition}
Two subsets $A, B\in \wp(S)$ will be said to be \emph{pg-roughly equal} (in symbols, $A\approx_{pg} B$) or \emph{acpg-roughly equal}) ($A\approx_{acpg} B$) as they respectively satisfy the conditions \textsf{pg-ro eq} or \textsf{acpg-ro} respectively:
\begin{align*}
(A^{l_\pi}, \mathbf{Sg}(A^{l_\pi}), A^{u_\pi}) = (B^{l_\pi}, \mathbf{Sg}(B^{l_\pi}), B^{u_\pi})    \tag{pg-ro eq}\\
(\mathbf{Sg}(A^{l_\pi}), A^{u_\pi}) = (\mathbf{Sg}(B^{l_\pi}), B^{u_\pi})    \tag{acpg-ro eq}
\end{align*}
\end{definition}

\begin{proposition}
Both $\approx_{pg}$ and $\approx_{acpg}$ are equivalence relations on $\wp(S)$, and \[\approx_{pg} \subseteq \approx_{acpg}  \]  
\end{proposition}

\begin{paragraph}{Problem} 
Algebraic systems definable over $\mathcal{PG}(S)$, $\mathcal{ACP}(S)$, $\wp(S)|\approx_{pg}$, and 
$\wp(S)|\approx_{acpg}$ are all of natural interest. What are the best possible ones? 
\end{paragraph}

$\mathcal{ACP}(S)$ is obviously a subset of $(Su(S))^2$. So the algebraic lattice order on the latter induces an order on the former. This is explored in section \ref{pialg1}. 

The upper approximation $u_\pi$ is arguably not the best, and a good contender is defined next. 

\begin{definition}
In the context of definition \ref{pigra}, define the \emph{anti-lower upper approximation} $u_a$ of a proper subset $A$ of $S$ by (the anti-lower upper approximation of $S$ will be taken to be $S$)
\begin{equation*}
A^{u_a} =  \Upsilon{H: A\subset H \,\&\, H\in Su(S)} \tag{a-upper}
\end{equation*}
\end{definition}

It is easy to see that repeated a-upper approximations will properly contain the set being approximated unless it is $S$. In fact, it is provable that

\begin{theorem}\label{aup}
In the context of theorem \ref{pi9}, all of the following hold (for any two subsets $A$ and $B$ of a pi-groupoid $S$):
\begin{align*}
A^{l_{\pi}}\subseteq A \subseteq A^{u_{\pi}}\subseteq A^{u_a}   \tag{+piInclusion}\\
A^{u_{a}} \subseteq A^{u_{a} u_{a}}    \tag{uapIn}\\
A^{l_{\pi}} \subseteq A^{l_{\pi} u_{a}}    \tag{luaIn}\\
A^{u_{a} l_{\pi}} = A^{u_{a}}   \tag{ulaId}\\
(A\subseteq B \longrightarrow A^{u_{a}} \subseteq B^{u_{a}})   \tag{uamo}\\
A^{u_{a}} \cup B^{u_{a}} \subseteq (A\cup B)^{u_{a}}   \tag{uaadd}\\
{\emptyset}^{l_{\pi}} = \emptyset \subseteq {\emptyset}^{u_{a}}   \tag{abottom}\\
S = S^{u_{a}} \tag{atop}
\end{align*}
\end{theorem}

\section{Abstract Example}\label{absexample}

An abstract example is constructed in this section for illustrating various aspects of up-directed approximation spaces. 

Let $\underline{S}$ be the set \[\underline{S} \,=\, \{a, b, c, e, f\} \text{ and let }\]
\[R \, =\, \{ac, bc, cc, af, ff, bf, ef, ca, cb, eb, cf, ea, fa, fb \}\] be a binary relation on it ($ac$ \textsf{means the ordered pair} $(a, c)$ and so on for other elements). In Figure \ref{udr}, the general approximation space $S = \left\langle \underline{S}, R \right\rangle$ is depicted. An arrow from $e$ to $f$ is drawn because $Ref$ holds. $\mathbf{c}$ and $\mathbf{f}$ are in bold because they are also related to themselves.

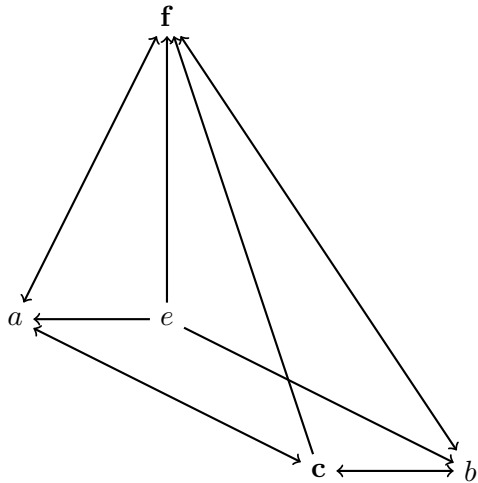
\begin{figure}[hbt]
\centering
\begin{tikzpicture}[node distance=2cm, auto]
\node (F) {$\mathbf{f}$};
\node (O) [below of=F] {};
\node (E) [below of=O] {$e$};
\node (K) [below of=E] {};
\node (C) [right of=K] {$\mathbf{c}$};
\node (A) [left of=E] {$a$};
\node (B) [right of=C] {$b$};
\draw[->,font=\scriptsize,thick] (E) to node {}(F);
\draw[->,font=\scriptsize,thick] (E) to node {}(B);
\draw[->,font=\scriptsize,thick] (E) to node {}(A);
\draw[<->,font=\scriptsize,thick] (F) to node {}(A);
\draw[->,font=\scriptsize,thick] (C) to node {}(F);
\draw[<->,font=\scriptsize,thick] (B) to node {}(F);
\draw[<->,font=\scriptsize,thick] (A) to node {}(C);
\draw[<->,font=\scriptsize,thick] (B) to node {}(C);
\end{tikzpicture}
\caption{Up-Directed Relation $R$}\label{udr}
\end{figure}

The up-directed approximation space $S = \left\langle \underline{S}, R \right\rangle$ is not reflexive and $R$ is not antisymmetric. The set of $R$-upper bounds for pairs of elements are 
given in the symmetric table \ref{ubs}:
\begin{table}[h]
\centering
\caption{Upper Bounds of Pairs}\label{ubs}
\begin{tabular}{|c|c c c c c|}
\toprule
 & a & b & c & e & f \\
\midrule
a & \{c, f\} & \{c,f\} & \{c,f\} & \{f \}& \{f\} \\
b & \{c, f\} & \{c, f\} & \{c\} & \{f\} & \{f\} \\ 
c & \{c, f\} & \{c\} & \{a, b, c, f\} & \{a, f\} & \{a, b, f\} \\
e & \{f\} & \{f\} & \{a, f\} & \{a, b, f\} & \{a, b, f\} \\
f & \{ f\} & \{f\} & \{a, b, f\} & \{a, b, f\} & \{a, b, f\} \\
\bottomrule
\end{tabular}
\end{table}

A groupoid corresponding to $S$ is given by Table \ref{grps} (hint for reading: row names multiply column names)
ac, bc, cc, af, ff, bf, ef, ca, cb, eb, cf, ea, fa, fb
\begin{table}[h]
\centering
\caption{A Groupoid of $S$}\label{grps}
\begin{tabular}{|c|c c c c c|}
\toprule
$\nearrow$ & a & b & c & e & f \\
\midrule
a & c & f & c & f & f \\
b & f & f & c & f & f \\ 
c & a & b & c & a & f \\
e & a & b & a & b & f \\
f & a & b & f & a & f \\
\bottomrule
\end{tabular} 
\end{table}

The subgroupoids of this groupoid are $\emptyset$, $\{c\}$, $\{f\}$, $\{a, c\}$, $\{b, f\}$, $\{c, f\}$, $\{e, f, b\}$, $\{a, c, f \}$,  $\{b, c, f\}$, $\{a,b, c, f\}$ and $S$.

The neighborhood granules determined by the elements of $S$ are as in Table \ref{nbdg}. 

\begin{table}[h]
\centering
\caption{Neighborhood Granules}\label{nbdg}
\begin{tabular}{| c | c | c | c| c| c|}
\toprule
$x $ & $ a $ & $ b $ & $ c $ & $ e $ & $ f $\\
\midrule
$[x]$ & $ \{ e, f\} $ &  $  \{c, e, f\} $ &  $\{a, b, c \} $ &
$ \emptyset $ &  $ S$ \\
\bottomrule
\end{tabular}
\end{table}

The set of subsets of $S$ closed under up-directedness is 
\begin{multline}
\mathcal{G} =\{\{c\}, \{f\}, \{a, c\}, \{b, c\},\{f,c\},\{b, f\}, \{e,f\}, \{a, f\},\\
\{a, c, f\},\{b, c, f\}, \{c, e, f\}, \{b, e, f \},\{a, e, f\}, \{a, b, f\},\{a, b, c\},\\
\{a,b, c,f\}, \{a, b, e, f\}, \{a, c, e, f\},\{b, c, e, f\}, S \} 
\end{multline}

Since $\wp(S)$ has $32$ elements, approximations of specific subsets are alone considered next.

Let $A =\{e, b, c\} $, then its approximations are as below:
\begin{align*}
A^l = \emptyset \text{ and } A^u = S \tag{nbd}\\
A^{l_{cd}} = \{b, c\} \text{ and } A^{u_{cd}} = \{b, c, e, f\} \tag{cd}\\
A^{l_\pi} = \{c\} \text{ and } A^{u_\pi} = S \tag{pi}
\end{align*}

Note however that if $B= \{b\}$, then $B^{l_\pi} = \emptyset$, $B^{u_\pi} = \{b, c, f\}= B^{u_a}$, $A^{u_a}= S$. and 
$B^{u_a} = \{\}$

\section{Pi-Groupoidal Algebraic Semantics}\label{pialg1}

Universal algebraic properties motivate the following approach. If $A$ and $B$ are two subsets of $S$, then let $A*B = \cup\{H: H\subseteq A\cap B \, \&\, H\in Su(S)\}$, and $A^{\flat}= \cup \{B:\, B\in Su(S) B\subseteq S\setminus A \}$

\begin{definition}\label{acpop}
On $\mathcal{ACP}(S)$, let the operations $\sqcup$, $\sqcap$, $\Cap$, $\Cup$, and $\neg$ be defined as follows for any $A, B\in\mathcal{ACP}(S) $ ($\wedge$ and $\vee$ are the lattice operations on $Su(S)$, and the subscripts indicate the first and second components)

\begin{align*}
A\sqcup B = (\mathbf{Sg}(A_1 \cup B_1), A_2\vee B_2))   \tag{acp-sum}\\
A \sqcap B = (\mathbf{Sg}(A_1 *B_1), A_2\wedge B_2 )   \tag{acp-pro}\\
\neg A = (\mathbf{Sg}(A_2^{\flat}), \mathbf{Sg}(A_1^{\flat}) )   \tag{acp-neg}\\
\coprod (A) =(A_1^{u_\pi}, A_2^{u_\pi})    \tag{acp-u}
\end{align*}
The algebra $\left\langle \mathcal{ACP}(S), \sqcup, \sqcap, \neg, \bot, \top  \right\rangle$ with $\bot = (\emptyset, \emptyset)$ and $\top = (S, S)$ will be called the \emph{ACP-rough algebra} over $S$. It will be endowed with the induced order given by 
\[(A_1, A_2) \unlhd (B_1, B_2) \leftrightarrow A_1 \subseteq B_1 \& A_2\subseteq B_2\]
 
\end{definition}

Note that the last defining condition in the above actually requires that $A_1$ be a subgroupoid of $B_1$ and $A_2$ be a subgroupoid of $B_2$.

\begin{proposition}
In the context of definition \ref{acpop}, all of the operations are well-defined.  
\end{proposition}
\begin{proof}
It suffices to show that the expressions on the right are in $\mathcal{ACP}(S)$ as they are unique.

$\mathbf{Sg}(A_1\cup B_1) $ is the groupoid generated by $A_1 \cup B_1$. So it is uniquely defined.
$A_2\vee B_2$ is a subgroupoid that contains the subgroupoid $\mathbf{Sg} (A_1\cup B_1) $ because by definition the first components are subgroupoids of the second components. Therefore, $A\sqcup B$ is well-defined. 
 
$A_1$ is a subgroupoid of $A_2$, yields $A_2^{\flat}$ is a subset of $A_1^{\flat}$. So $\neg A$ is well-defined. 
\end{proof}

\begin{theorem}
In the context of definition \ref{acpop}, all of the following hold: 
\begin{align}
\left\langle \mathcal{ACP}(S), \sqcup, \sqcap, \bot, \top  \right\rangle \text{ is an algebraic lattice.}     \tag{A1}\\
A \unlhd \neg \neg A    \tag{A2}\\
A \unlhd B \longrightarrow \coprod A \unlhd \coprod B   \tag{A3}\\
A \unlhd \coprod A   \tag{A4}\\
A \unlhd B \longrightarrow \neg B \unlhd \neg A   \tag{A5}\\
\neg\coprod\neg A \unlhd \coprod A \tag{A6}
\end{align}
\end{theorem}
\begin{proof}
The direct product of two algebraic lattices is an algebraic lattice and a sublattice of the latter is also an algebraic lattice. Therefore $\left\langle \mathcal{ACP}(S), \sqcup, \sqcap, \bot, \top  \right\rangle$ is an algebraic lattice. 

$\neg A$ by definition is $(\mathbf{Sg}(A_2^{\flat}), \mathbf{Sg}(A_1^{\flat}) )$. The first component is the subgroupoid generated by the union of subgroupoids in $S\setminus A_2$.  Because $S\setminus A_2$ is a subset of $S\setminus A_1$, the subgroupoid generated by the former is contained in that generated by the latter. Now a second negation of $A$ would result in the form  $\mathbf{Sg}((\mathbf{Sg}(A_1^{\flat}))^{\flat}), \mathbf{Sg}((\mathbf{Sg}(A_2^{\flat}))^{\flat})) )$. It can be seen that $A_i \subseteq \mathbf{Sg}((\mathbf{Sg}(A_i^{\flat}))^{\flat})$ for $i=1,2$ because $X \subseteq \mathbf{Sg}(X)$ for any subset $X$. The reader may represent this proof visually.

Suppose $A\unlhd B$, then $A_1 \subseteq B_1$ and $A_2 \subseteq B_2$. $u_\pi$ is a monotonically increasing operator. Therefore, $A_1^{u_\pi} \subseteq B_1^{u_\pi} $ and $A_2^{u_\pi} \subseteq B_2^{u_\pi} $. From this it follows that $\coprod A \unlhd \coprod B$.

\textsf{A4} follows from the definition and basic properties of $u_\pi$.

If $A\unlhd B$, then $\neg B$ is the pair of subgroupoids generated by the union of subgroupoids of the form $S\setminus B_i$ (for $i=2, 1$). $S\setminus A_i$ is a superset of $S\setminus B_i$ (for each $i$). So the union of subgroupoids contained in the former would also be a superset of the union of subgroupoids contained in the latter. This ensures \textsf{A5}.
\end{proof}

If $S$ is finite, then the lattices are also atomistic (or even when the lattices are atomistic) and the decomposition into atomistic algebraic lattices are of interest \cite{llib1995}. This may lead to interesting representations in specific application contexts.

\begin{paragraph}{Problem}
What additional assumptions ensure a duality result in the context?  
\end{paragraph}

\section{Meanings}

In general, the groupoid operation can be read in at least three ways. The operation obviously adds information to the general approximation space -- \emph{this addition can be read as a decision because it involves choice among alternatives}. In fact, the collection of all possible groupoidal operations can be used to generate a decision space. Though these are always present in the original relation, the main issue here is of an accessible abstraction towards reasoning with the same. As such this aspect can be investigated in the given form or by taking the exact region to which the result of the operation belongs relatively. For the latter perspective, the groupoidal operation over $\wp (S)$ can be read as a combination of operations that are relatively better behaved relative to the approximations, aggregation and commonality operations. This permits easier interpretation, and semantics. In addition to these it is possible to interpret the up-directed aspect as a property satisfied by a set of concepts. Thus it would be a description of the  \emph{organization of a set of concepts}.

Specialization of these aspects to the specific approximations considered earlier, and pi-groupoids is motivated by many application contexts, and leads to distinct interpretations. In all cases, the following definition concerning regions determined by a pair of subsets can be reinterpreted in relation to associated granules. The semantic operations studied in previous sections can represent related decision regions in a better way.

\begin{definition}
For any $A, B\in \wp(S)$, the following operations can be defined:
\begin{equation*}
n(A, B) \,=\, \{b:\, (\exists a\in A \exists b\in B)\, ab =b\} \tag{normal}
\end{equation*}
\begin{equation*}
o_1(A, B) \,=\, \{c:\, (\exists a\in A \exists b\in B)\, ab =c\in U_R(a, b)\setminus A \} \tag{outer-1}
\end{equation*}
\begin{equation*}
o_2(A, B) \,=\, \{c:\, (\exists a\in A \exists b\in B)\, ab =c\in U_R(a, b)\setminus B \} \tag{outer-2}
\end{equation*}
\begin{equation*}
i_1(A, B) \,=\, \{c:\, (\exists a\in A \exists b\in B)\, ab =c\in U_R(a, b)\cap A \} \tag{inner-1}
\end{equation*}
\begin{equation*}
i_2(A, B) \,=\, \{c:\, (\exists a\in A \exists b\in B)\, ab =c\in U_R(a, b)\cap B \} \tag{inner-2}
\end{equation*}
\begin{equation*}
 o(A, B) = o_1(A, B)\cap o_2(A, B) \tag{outer}
\end{equation*}
\end{definition}

In the above definition, the global groupoid operation has been split into multiple operations based on the relative values assumed. For any two sets $A, B\in \wp(S)$, 
\begin{itemize}
\item {$n(A, B)$ is the set of things in $B$ that have some part or approximate part in $A$,}
\item {$o_1(A, B)$ is the set of things in the outer core determined by elements of $A\times B$ that are not in $A$,}
\item {$o_2(A, B)$ is the set of things in the outer core determined by elements of $A\times B$ that are not in $B$,}
\item {$i_1(A, B)$ is the set of things determined by elements of $A\times B$ that are in $A$,}
\item {$i_2(A, B)$ is the set of things determined by elements of $A\times B$ that are in $B$, and}
\item {$o(A, B)$ is the set of things determined by elements of $A\times B$ that are not in $A$ or $B$.}
\end{itemize}

These lead to interesting operations on the decisions and are described in a separate joint paper by the present author.

\subsection{Organization of Complex Concepts}

When concepts are derived within a system that are closed under consequence or when they satisfy a number of mutual consistency properties then they may be representable within lattice-theoretical frameworks of general rough sets and formal concept analysis. Additionally, logics associated with such systems are known. In practice, such classification is often not feasible, and mutual inconsistency may be present. Circular reasoning happens even in philosophical contexts, where practitioners demonstrate high levels of engagement with their discourse. In situations where people are constrained by time to limit their engagement with information from multiple sources, and relative to their own perspective/biases about the nature of information in the constraint, clear order structures may not emerge. The connections are complicated additionally by the influence of the environment on the information generated by the sources in question -- this is an essential part of the distributed cognition paradigm .

In the context of time-constrained evaluation by experts or teachers, it is necessary for 
expressions generated by non-experts or students to be approximated. Additionally, multiple sets of expressions by non-experts or students may be compared in the functional language of the expert or teacher (as opposed to the object-level language of the subject matter) \cite{am2022c,bdtmp2008,jsr2022}. These acts result in up-directed approximation spaces or specialized variants like pi-groupoids. In comparison to the neighborhoods, and granules used in previous papers by the present author \cite{am5550,am5559,ams2020}, \emph{the ones introduced and studied here are more appropriate because their closure properties are more tightly coupled with consequence}. Closure under up-directedness of a granule ensures that apparent higher-level concepts are all within the granule -- this helps in discerning things that are not subsumed by the concepts.

\subsection{Knowledge in Rough Sets}\label{seckno}

In a general approximation space $S$, if $R$ is an equivalence, a partial order or a quasi order, then it is additionally possible to associate other groupoidal operations (see \cite{am501,am5019,am909}) on $S$. This is discussed in brief in Sec.\ref{apprsp}. However, the associated operation is distinct from the one considered in this paper. 

General and classical rough sets have been associated with concepts of knowledge and studied from that perspective in a number of papers by the present author \cite{am9114,am9006,am9501,am9969,am99} and others\cite{zpb,zp6,ppm2,chp3,bgc12}. The basic idea in the context of classical approximation spaces \cite{zpb} is to associate definite objects with concepts and consequently the equivalence relation $R$ is  associated with knowledge. In more general situations, granularity has a bigger role to play, and knowledge is defined relative to granular axioms used and other desirable properties. Examples of such conditions are
\begin{description}
\item [GK1]{Individual granules are atomic units of knowledge.}
\item [GK2]{If collections of granules combine subject to a concept of mutual
independence, then the result would be a concept of knowledge. The 'result' may
be a single entity or a collection of granules depending on how one understands the
concept of \emph{fusion} in the underlying mereology.}
\item [GK3]{Maximal collections of granules subject to a concept of mutual independence are admissible concepts of knowledge.}
\item [GK4]{Parts common to subcollections of maximal collections may be interpreted as knowledge.}
\item [GK5]{All stable concepts of knowledge consistency should reduce to correspondences between granular components of knowledges. In particular, two relations $R_1$ and $R_2$ may be said to be \emph{consistent} if and only if the set of granules associated with the two general approximation spaces have bijective  correspondence. }
\end{description}

The granular knowledge axioms mentioned above are not related to the groupoidal axioms directly. \emph{This is because both CUD- and pi-groupoidal operations provide additional layers of decision making that need be to integrated with existing work}. The application contexts in \cite{am5550,am5559,ams2020} and the following sections support this claim.

\section{Applications to Human Learning}

Applications of the theoretical approaches invented in this research to modeling student centered learning contexts for the purpose of {analysis, digital learning platforms, creation of dynamic teaching aids, and intelligent evaluation} are of much interest. In \cite{am5550,am5559}, applications to modeling classroom activities using GGSp is proposed by the present author. In a forthcoming paper \cite{ams2020}, this is considered in the context of directed parthood spaces. However, The application to concept inventories for evaluation in \cite{am5559} is in a relatively structured context, and so can be formulated in terms of parthood spaces under additional assumptions.

In student-centered learning students are put at the center of the learning process, are encouraged to learn through active methods, and typically become more responsible for their learning in such environments. In traditional teacher-centered classrooms, teachers have the role of instructors and are intended to function as the only source of knowledge. By contrast, teachers are typically intended to perform the role of facilitators in student-centered learning contexts. Some best practices for teaching in such contexts are described in \cite{jrp2016}. Teachers need to constantly improve their methods in such teaching contexts because that is part of the methodology. 

Models of the content, teaching context and language of the subject can be assumed to be sufficiently expressive. For example, the language used by students and teachers to express concepts, predicates, and predications in mathematics is frequently different from possible candidates of a standard language. These are approximated across semantic domains with partial success by them. In \cite{am2022c}, it is argued by the present author that mereology combined with a language of approximations (which is compatible with the intended language of this paper) can potentially be used to build higher order formalizations of concepts that go beyond the restrictions envisaged in \cite{jsr2022} and \cite{bcnp2016} because the very idea of the formal is much broader in the former through its’ accommodation of vagueness and uncertainty.  

Because of the open-ended aspect of the learning process, it is not expected that teachers have absolute control over the concepts learned. Students may themselves arrive at new methods of solution or define new concepts as part of the learning process. In this scenario it is of interest to suggest potential higher concepts that relate to the progress of the work in question. Teachers can possibly provide some initial suggestions and subsequently these can be worked upon by algorithms relying upon datasets of concepts for improved suggestions. From the perspective of this research this becomes the problem of construction of the best groupoid operations.

In more precise terms,
\begin{description}
\item [L1]{Let $A$ and $B$ be two concepts arrived at by the learner. The open-ended nature of the learning process means that a general rough set model of concepts must be adaptive or permit supervision. A concept lattice of knowledge would be too far fetched or impractical in the situation.}
\item [T1]{Teacher observes that a set of concepts $C$  contains $A$ and $B$ in some sense, and offers suggestions relating to the scenario based on selections from $C$. But a sufficiently large set of teachers need to be considered.}
\item [S1]{Software aid for the learning context provides better suggestions based on \textsf{L1} and \textsf{T1} using a groupoidal decision model instead of the former alone. In general available strategies that can be used to arrive at suggestions based on \textsf{L1} alone are likely to be unintelligent.}
\end{description}

\emph{The very idea of admitting a general approximation space means that a collection of concepts and attributes have been identified. However, this by itself does not guarantee the construction of lattices of concepts}. \textsf{L1} should be read in relation to this.

It may be noted that the impact of AI on enhancing classroom learning and learning in general has been very limited (see \cite{chos2018} and related references). In fact digital technology in the context of mathematics teaching has been stagnating because most of the effort has been on non-intelligent software that merely aid communication. There is no dearth of motivation for such work -- Often teachers do not have sufficient knowledge about the working of their students mind, have an excess of work load at hand and may be suffering from cognitive dissonances of specific types. 

In a forthcoming book by the present author, the rough methodology suggested in this subsection is applied to specific practices such as opening of exercises in the context of mathematics teaching \cite{sko2011,milani2019}, use of explicit mathematical language \cite{usz2012}, and software for student expression \cite{alp2019,chos2018}.

Learning through simulations may or may not be supplemented with rigorous proofs. Some proofs may be partial, and yet students may be able to explore a spectrum of new connections that may require relatively advanced understanding through traditional proof based methods. This is explored in the next subsection.

\subsection{Modeling Learning with Geogebra}

Geogebra is a versatile free software for learning a wide variety of mathematics including calculus through simulations and dynamic models. It is used in classrooms all over the world and has a vibrant teacher and research community that have contributed to its development and use (to the point that there are a few Geogebra institutes in many parts of the world). Aspects of the structure of concepts in typical online or offline classroom settings is explored here. The reader is referred to \cite{lbrs2011,gven2013,mapusa2018} for varied uses of the software. The mereological aspect is also stressed in \cite{am2022c} by the present author.

While transformational geometry is almost natural in Geogebra, an exercise (from \cite{gven2013}) like the following illustrates one dynamic way of doing geometry: 
\begin{quote}
$3.2.2.$ Make a new sketch that illustrates the Euler line theorem. First construct a triangle and then use the tools from $3.1$ to create the three triangle centers. Hide any intermediate objects so that only the triangle itself and the three triangle centers are visible. Construct a line through two of the 
points and observe that it additionally passes through the third. Add hide-show check boxes for the triangle, the three centers, and the Euler line. Add text boxes and check boxes for the definitions of the three triangle centers as well as the statement of the Euler line theorem. Make sure that you have made good use of color and explanatory text so that your sketch is user friendly.
\end{quote}

\emph{This flexibility naturally permits the creation of unexpected concepts that differ in value to the point that evaluation becomes difficult and error-prone especially because of limited time and the absence of sufficiently skilled teachers (which is often the case in a number of settings)}. So there is a need to help teachers and with additional soft tools. 

In addition, a formal approach can help in specifying directions of concept learning. For example, a student can learn about aspects of Pythagoras theorem, the connection between circles drawn with the hypotenuse of a right-angled triangle, and  circum-circles of such triangles without ever knowing much about the proof proper.

Learning from simulations and a finite set of perceivable models has its limitations. It is moreover prone to paradoxes, and visual proofs may actually be wrong. Besides students and teachers are likely to have their own ideas of which concept is to be regarded as part of something else. When such information is labeled and combined using reasonable Boolean combinations, the resulting part of relation is likely to satisfy fewer properties. For example, if $\pc_i$ are part-of relations specified by the $i$ th entity then $\pc$ can be defined by a formula of the form (for some $\Phi$)
\[\Phi (\pc_1 ab, \, \pc_2 ab\, \ldots \, \pc_n ab) \leftrightarrow \pc ab \]

The groupoidal operation and directedness can be used to specify optimal choices of concept progression relative to $\pc$.

\section{Pi-groupoidal Rough Clustering and Validation}

CUD and Pi-groupoidal-rough clustering is introduced as a new method of soft clustering and it's use for soft cluster validation along the lines invented in \cite{am2021c} by the present author is discussed in this section. Empirical aspects associated with the application context will appear separately.

The following minimalist definitions are proposed on the basis of the application contexts and \cite{am2021c}. 
Often it may be useful to require that the upper approximations of boundaries of clusters are small in some sense. But these are not imposed in the definitions.

\begin{definition}
By a \emph{cud-rough clustering} of an up-directed approximation\\ space $S$ will be meant a collection $\mathcal{S} = \{(A^{l_{cd}}, A^{u_{cd}}, A^{b_{cd}}): \, A\subseteq S\}$ that satisfies ($A\in_{c} \mathcal{S}$ is intended to mean that $(A^{l_{cd}}, A^{u_{cd}}, A^{b_{cd}}) \in \mathcal{S}$)
\begin{align*}
\bigcup_{A\in_c\mathcal{S}} A^{l_{cd}} = S \tag{$l_{cd}$-cover} \\
A, B \in_c \mathcal{S} \longrightarrow A\nsubseteq B \,\&\, \neg (A\approx_{cd} B)  \tag{cd-disclusion}\\
\end{align*}
\end{definition}

\begin{definition}
By a \emph{pi-rough clustering} of an up-directed approximation space $S$ will be meant a collection $\mathcal{S} = \{(A^{l_\pi}, A^{u_{\pi}, A^{b_\pi}}): \, A\subseteq S\}$ that satisfies ($A\in_{c} \mathcal{S}$ is intended to mean that $(A^{l_\pi}, A^{u_\pi}, A^{b_\pi}) \in \mathcal{S}$)
\begin{align*}
\bigcup_{A\in_c\mathcal{S}} A^{l_\pi} = S \tag{$l_\pi$-cover} \\
A, B \in_c \mathcal{S} \longrightarrow A\nsubseteq B \,\&\, \neg (A\approx_{\pi} B) \tag{disclusion}
\end{align*}
\end{definition}

\subsection{Clustering Sentinel Data}

The synthetic aperture radar (SAR) technique \cite{gj2018} typically involves multiple sensors for the generation of high resolution spatial images from both airborne and space borne platforms. Resultant images need to be processed further after  denoising for proper interpretation. In image segmentation and classification of SAR satellite images of inaccessible regions, ground truths may not be known at all. For example, the images themselves may be the main source of information of heavily forested regions (that are not effectively navigable). A wide variety of machine learning techniques including soft clustering are used to study such images. 

For example, the Copernicus Sentinel-2 mission of the European space agency \cite{gj2018} is based on a constellation of two identical satellites in the same orbit. Each satellite carries a high-resolution multispectral imager. This sensor delivers 13 spectral bands ranging from 10 to 60-meter pixel size. Its blue (B2), green (B3), red (B4), and near-infrared (B8) channels have a 10-meter resolution. Next, its red edge (B5), near-infrared NIR (B6, B7, and B8A), and short-wave infrared SWIR (B11 and B12) have a ground sampling distance of 20 meters. Finally, its coastal aerosol (B1) and cirrus band (B10) have a 60-meter pixel size. A number of indices that correspond approximately to regions with specific characteristics are computable from related datasets. The correspondence between region-type identifying indices and qualities of a region are only approximately understood by remote-sensing experts, and this is reflected in the very large number of competing formulas proposed for computing the same. Further, the exact information extractable from satellite images is not fully understood.

Datasets derived from similar sensors have $b$ number of columns (corresponding to bands, and position) and $n$ rows (corresponding to regions).  
While $b$ is a small integer (Sentinel-2 image datasets have 13 band and 2 position indicating columns), $n$ can be of the order of $10^7$ or higher (depending on the size of the region). Cells have intensity values in the range $1-5600+$. 

The following new methodology (defined step-wise) for segmentation, and identification of features of Sentinel-2 images (and other satellite images as in \cite{sstk2017}) is proposed and applied by the present author in a separate joint paper.  

\begin{description}
\item[Step-1] {Define a up-directed relation $R$ as follows (this step is data-driven). For any two row vectors $a$ and $c$ ($\rho$ being a suitable distance, and $\epsilon$ depending on the region-type of $a$ and $c$), let \[Rac \leftrightarrow 0\leq (c-a) \& \rho(a, c) \leq \epsilon\]}
\item[Step-2] {Define a pi-groupoidal operation, compute pi-groupoidal approximations, boundaries, and complements of pi-upper approximations. Tuples of the form $(A^{l_\pi}, A^{u_\pi}, A^{b_\pi})$ are potential rough clusters. Note that not all approximations are of interest. }
\item[Step-3] {Compute the variance or standard deviation of each of the indices associated with components in step-2.}
\item [Step-4] {Alternatively, compute the \textsf{normalized average squared distance within each component of the tuples}. }
\item[Step-5] {Select the rough clusters on the basis of step-$3$ or step $4$, and order of priority of components (bands).}
\end{description}

The main advantage is that implementations have lesser computational load, sampling methods are involved in the definition of the pi-groupoidal operation, and stratification is somewhat automatic and local.  Step-2 does yield approximations, decision regions, and rough objects, however it need not be fully coherent with computed indices. For example, it maybe desirable to treat two water bodies with distinct vegetation indices as separate entities (because of the amount of vegetation or chlorophyll containing material that they have), and approximations that include both may not be of interest. The geometrical aspect of images (in terms of latitude and longitude) may also be needed to be coupled with ideas of of a sub region being next to or overlapping another sub region in the selection -- this is because physically disparate regions are more likely to have distinct ecological influences.

The methodology constitutes a new rough clustering method and will be referred to as the \emph{pi-groupoidal rough clustering}. Since ground truth for the dataset is not available, the method is also applied to that of \cite{pbad2019}, and compared. While the latter has some ground truth, the reductionist approach of comparing forests in one place with forests elsewhere remains questionable.  Further the quality of soft clustering by another tolerance-based constructive rough clustering, fuzzy K-means and variants is again directly compared by correspondence. 

Segmentation of the image is implemented by juxtaposing the clustering information with the positional part of the original data set (in terms of latitudes and longitudes).

The above is a constructive semi-supervised method of clustering that can be applied to similar contexts. It is especially useful in situations where the relation (of containment in a sense) of given objects to two other objects can be vaguely indicated. In the above situation, various controversial indices that are reliable in some regions help in such decision making. The definition of groupoidal or pi-groupoidal operations in the context can be guided by external measures or by a form \emph{directed local sampling}. This can be significantly better than a stratified sampling over the whole dataset as for the satellite data considered.

\section*{Further Directions and Remarks}

The concepts of up-directed and up-directed approximation spaces invented in \cite{am5550} and a separate paper are studied in detail from the perspective of closure under up-directedness and related groupoids in this research. 
CUD-approximations are guided by closure under up-directedness -- a feature of sets of concepts that are relatively closed under consequence in a perspective. These operators turn out to be cautious monotonic, and therefore may be associated with non-monotonic reasoning (see \cite{dm94,dm2003}). These are studied, and applied to meaningful approximate organization of concepts in pedagogic content knowledge (and related modeling). It is also shown that CUD-approximations satisfy weaker properties in comparison to pi-groupoidal approximations. Modal connections are harder because complementation does not help. Other negations may however be considered. In the study on rough sets over partial and quasi orders, a similar closure has not been considered, and can be useful area of research.

It is shown that pi-groupoidal approximations have the most algebraic approach to semantics ever in the entirety of the algebraic approaches to rough sets. This is because of the use of algebraic closure operators in the definition of the pi-groupoidal approximations itself. However, the minimal assumptions restrict additive and multiplicative properties. Generalized orders such as $\lambda$-lattices, graphs, hyper-graphs, and partial lattices are used in rough sets (see \cite{am501,gcctf2017}), and extensions of related studies with up-directedness is of natural interest.

Comparison of the granular approach used in this paper with other granular approaches will appear separately. Meaningful construction of clusterings is achieved through the introduced semi-supervised method of rough clustering. Specializations of the method can be expected to be applicable to contexts over which investigators have good insight into the issues, and concepts involved.

\begin{flushleft}
\textbf{Acknowledgment:} 
\end{flushleft}
\begin{small}
This research is supported by woman scientist grant No. WOS-A/PM-22/2019 of the Department of Science and Technology, India.
\end{small}

\bibliographystyle{fundam}
\bibliography{algroughf69flzf.bib}

\end{document}